\documentclass[sigconf]{aamas}

\usepackage{balance} 

\acmConference[]{Preprint}{Saskatchewan}{Canada}{University of Regina, Regina}

\usepackage{comment}
\usepackage{tikz}
\newcommand{\defi}{\doteq}
\newtheorem{axiom}{Axiom}
\newtheorem{lemma}{Lemma}

\title[On the Semantics of Primary Cause in Hybrid Dynamic Domains]{On the Semantics of Primary Cause in Hybrid Dynamic Domains}

\author{Shakil M.\ Khan}
\affiliation{
  \institution{University of Regina}
  \city{Regina, Saskatchewan}
  \country{Canada}}
\email{shakil.khan@uregina.ca}

\author{Asim Mehmood}
\affiliation{
\institution{University of Regina}
  \city{Regina, Saskatchewan}
  \country{Canada}}

\author{Sandra Zilles}
\affiliation{
\institution{University of Regina}
  \city{Regina, Saskatchewan}
  \country{Canada}}
\email{sandra.zilles@uregina.ca}

\begin{abstract}
Reasoning about actual causes of observed effects is fundamental to the study of rationality. This important problem has been  studied since the time of Aristotle, with formal mathematical accounts emerging recently. We live in a world where change due to actions can be both discrete and continuous, that is, hybrid. Yet, despite extensive research on actual causation, only few recent studies looked into causation with continuous change. Building on recent progress, in this paper we propose two definitions of primary cause in a hybrid action-theoretic framework, namely the hybrid temporal situation calculus. One of these is foundational in nature while the other formalizes causation through contributions, which can then be verified from a counterfactual perspective using a modified ``but-for'' test. We prove that these two definitions are indeed equivalent. We then show that our definitions of causation have some intuitively justifiable properties.
\end{abstract}


\newcommand{\BibTeX}{\rm B\kern-.05em{\sc i\kern-.025em b}\kern-.08em\TeX}

\setcopyright{none}
\begin{document}


\pagestyle{fancy}
\fancyhead{}


\maketitle 


\section{Introduction}
A fundamental challenge in dynamic domains is to identify the causes of change, in particular, that of an observed effect becoming true given a history of actions (the ``scenario''), a problem known as actual or token-level causation. Based on Pearl's foundational work \cite{Pearl98,Pearl00}, Halpern and Pearl (HP) and others \cite{Halpern00,HalpernP05,EiterL02,Hopkins05,HopkinsP07,Halpern15,Halpern16} have extensively studied the problem of actual causation and significantly advanced this field. HP adopted David Hume's reductionist viewpoint on causality and provided a rigorous formalization using structural equations models (SEM) \cite{Simon77}. Following Hume, HP defined causes using counterfactuals formalized via ``interventions''. Humean counterfactual definition of causation, that ``\emph{a cause to be an object, followed by another $\cdots$ where, if the first object had not been, the second never had existed}'',\footnote{The second-half of the quoted definition is also known as the ``but-for'' test.} suffers from the problem of preemption: it could still be the case that in the absence of the first object, the second would have occurred due to the presence of a third object, whose effects in the original scenario was preempted by the first's. In other words, it might be that in the original scenario, both the first and the third objects were competing to achieve the same effect (the second object), and the third failed to do so, as the first, which occurred earlier, had already achieved the effect. HP's formalization avoids preemption by performing selective counterfactual analysis 
under carefully chosen contingencies. While their inspirational work has been used for practical applications,  
their approach based on SEM has been nevertheless shown to fail for some examples and was criticized for its limited expressiveness \cite{Hopkins05,HopkinsP07,GlymourDGERSSTZ10}. To deal with this, others have expanded HP's account with additional features, such as temporal logic \cite{GladyshevADD025} and more general forms of counterfactuals that are non-interventionist in nature \cite{Aguilera-Ventura25}. Yet others have proposed alternate formalizations of actual cause from first principles, e.g.\ using the notions of \emph{counterfactual dependence} and \emph{production} \cite{BeckersV18} and from a non-counterfactual ``regularity’’ perspective \cite{Mackie65}, e.g., using propositional non-monotonic logic frameworks \cite{Bochman18a, Bochman18b, Bochman03}.
\par
In recent years researchers have become increasingly interested in studying causation within more expressive action-theoretic frameworks, in particular in that of the situation calculus \cite{BatusovS17,BatusovS18,KhanS20}. Among other things, this allows one to formalize causation from the perspective of individual agents by defining a notion of epistemic causation \cite{KL21} and by supporting causal reasoning about conative effects, which in turn has proven useful for explaining agent behaviour using causal analysis \cite{KR23}. This line of work has recently been extended to deal with non-deterministic domains \cite{KLR25a,KLR25b} and multiagent synchronous concurrent games \cite{KKL25}, which led to the formalization of novel concepts such as causal responsibility \cite{KKL25tr}.\footnote{The term ``causal responsibility'' is also used to refer to strategic responsibility due to the nature of actions and their (causal) effects \cite{DeGiacomoLPP25,Vincent11}; in contrast, in \cite{KKL25tr} it is used to indicate responsibility due to \emph{actual causal contributions to an outcome}.} 
\par
A distinguishing feature of the real world is that change can be both discrete and continuous. However, despite the enormous body of work on actual causality, almost all have focused on defining causes in discrete domains. In an attempt to support continuous variables, Halpern and Peters \cite{HalpernPeters22} recently introduced an extention of SEM, called generalized structural-equations models. However, it is not clear how one can formalize various aspects of action-theoretic/dynamic frameworks there, e.g.\ non-persistent change supported by fluents, possible dependency between events, temporal order of event occurrence, etc. Also, an actual definition of causation was not given.
%
Recently in \cite{MK2024}, we proposed a preliminary definition of actual cause in a hybrid action theoretic framework. 
While this definition appeals to intuition, many argue that causation should be rather defined in counterfactual terms. Note that, there is strong experimental evidence that humans understand causes using counterfactual reasoning \cite{GerstenbergGLT15,GerstenbergGLT14}.
\par
To deal with this, building on previous work \cite{KL21}, in this paper we propose two definitions of actual cause in a hybrid action-theoretic framework. Our formalization is set within a recently proposed hybrid temporal situation calculus \cite{BatusovGS19a,BatusovGS19b}. We focus on actual primary cause and study causation relative to primitive fluents exclusively. One of our proposed definitions is foundational and refines \cite{MK2024}, while the other formalizes causation through contributions, which can then be verified from a counterfactual perspective using a modified but-for test. Namely, we show that if one were to remove the identified cause along with some contingencies --in particular, actions that are currently preempted by the actual cause-- from the given scenario, the outcome would no longer follow. We prove that the proposed definitions are indeed equivalent and that our definitions of causation have some intuitively justifiable properties.
%
%

\section{Preliminaries}\label{prelim}
\noindent\paragraph{\textbf{The Situation Calculus (SC)}}
The situation calculus is a well-known second-order language for representing and reasoning about dynamic worlds \cite{McCarthyH69,Reiter01}. In the SC, all changes 
are due to named actions, which are terms in the language. Situations represent a possible world history resulting from performing some actions. The constant $S_0$ is used to denote the initial situation where no action has been performed yet. The distinguished binary function symbol $\mathit{do}(a, s)$ denotes the successor situation to $s$ resulting from performing the action $a$. The expression $\mathit{do}([a_1,\cdots,a_n],s)$ represents the situation resulting from executing actions $a_1,\cdots,a_n$, starting with situation $s$. As usual, a relational/functional fluent representing a property whose value may change from situation to situation takes a situation term as its last argument. There is a special predicate $\mathit{Poss}(a,s)$ used to state that action $a$ is executable in situation $s$. Also, the special binary predicate $s \sqsubset s'$ represents that $s'$ can be reached from situation $s$ by executing some sequence of actions.  $s\sqsubseteq s'$ is an abbreviation of $s\sqsubset s'\lor s=s'$. $s < s'$ is an abbreviation of $s\sqsubset s'\land\mathit{Exec}(s')$, where $\mathit{Exec}(s)$ is defined as $\forall a',s'.\;do(a',s')\sqsubseteq s \supset\mathit{Poss}(a', s')$, i.e.\ every action performed in reaching situation $s$ was possible in the situation in which it occurred. $s \leq s'$ is an abbreviation of $s<s'\vee s=s'$.
\par
In the SC, a dynamic domain is specified using a \emph{basic action theory (BAT)} $\mathcal{D}$ that includes the following sets of axioms: (i) (first-order or FO) initial state axioms $\mathcal{D}_{S_0}$, which indicate what was true initially; (ii) (FO) action precondition axioms $\mathcal{D}_\mathit{ap}$, characterizing $\mathit{Poss}(a, s)$; (iii) (FO) successor-state axioms (SSA) $\mathcal{D}_\mathit{ss}$, indicating precisely when the fluents change; (iv) (FO) unique-names axioms $\mathcal{D}_{\mathit{una}}$ for actions, stating that different action terms represent distinct actions; and (v) (second-order or SO) domain-independent foundational axioms $\Sigma$, describing the structure of situations \cite{LevPirRei98}.
Although the SC is SO, Reiter \cite{Reiter01} showed that for certain type of queries $\phi$, $\mathcal{D}\models\phi$ iff $\mathcal{D}_{\mathit{una}}\cup\mathcal{D}_{S_0}\models\mathcal{R}[\phi]$, where $\mathcal{R}$ is a syntactic transformation operator called {\em regression} and $\mathcal{R}[\phi]$ is a SC formula that compiles dynamic aspects of the theory $\mathcal{D}$ into the query $\phi$. Thus reasoning in the SC for a large class of interesting queries can be restricted to entailment checking w.r.t a FO theory \cite{Reiter01}.
\noindent\paragraph{\textbf{Hybrid Temporal Situation Calculus (HTSC)}}
The SC only allows discrete changes to fluents as a result of actions. However in the real world, many changes are continuous 
and happen due to the passage of time. For example, a change in room temperature after adjusting the thermostat happens over time, but not immediately. Reiter's temporal SC \cite{Reiter01} can model continuous change. In his framework, each action is given a time argument, but the fluents remain atemporal and do not actually change over time; instead, they attain certain values when these time-stamped actions are performed. One cannot query the value of a continuous fluent at some arbitrary time without referencing a time-stamped action. 
\par
To accommodate time, Reiter introduced two special functions, $\mathit{time}(a)$, which refers to the time at which an action $a$ is executed, and $\mathit{start}(s)$, which gives the starting time of the situation $s$. $\mathit{time}$ is specified by an axiom $\mathit{time}(a(\vec{x},t))=t$ (included in $\mathcal{D}_{S_0}$) for every action function $a(\vec{x},t)$ in the domain. $\mathit{start}$ is specified by the new foundational axiom $\mathit{start}(do(a,s))=\mathit{time}(a).$ The starting time of $S_0$ is not enforced. To outlaw temporal paradoxes, the abbreviation $\mathit{Exec}(s)$ is redefined as: 
$\mathit{Exec}(s)\defi\forall a,s’.\;do(a,s’)\sqsubseteq s\supset(\mathit{Poss}(a,s’)\land\mathit{start}(s’)\leq\mathit{time}(a)).$
\par
The HTSC \cite{BatusovGS19a,BatusovGS19b} takes inspiration from hybrid systems in control theory, which are based on discrete transitions between states that continuously evolve over time. In HTSC, (atemporal) SC fluents are preserved, not to represent continuous change, but rather to provide a context within which the values of temporal fluents can change. For instance, the velocity of a ball, a continuous fluent, changes over time when the ball is in a state of falling. Here, the discrete fluent $\mathit{Falling}(s)$ serves as the context that influences how the continuous functional fluent $\mathit{velocity}(t,s)$ varies with time.
\par
In HTSC, the contexts are mutually exclusive to make sure continuous fluents do not assume two different values at the same time. This is enforced by 
the following condition 
in the background theory (the Mutex Axiom, henceforth):\footnote{Hereafter, free variables in a sentence are assumed to be $\forall$\!-quantified at the front.}
$\Phi(\overline{x}, y, t, s) \land \Phi(\overline{x}, y', t, s) \supset y = y'$.
%
HTSC modifies SC's BAT by including the axioms for $\mathit{time}(a)$ and $\mathit{start}(s)$ as well as the new definition of $\mathit{Exec}(s)$ 
in $\mathcal{D}_{S_0}$ and in $\Sigma$, respectively. Moreover, in addition to Reiter's SSA \cite{Reiter01}, which specify how fluents change as a result of named actions, HTSC introduces the following state evolution axioms (SEA) \cite{BatusovGS19a}, each of which defines how a temporal fluent $f(\vec{x})$'s value changes over time: 
%
$f(\vec{x}, t, s) = y \equiv [\Phi(\vec{x}, y, t, s) \lor y = f(\vec{x}, start(s), s) \land \neg\Psi(\vec{x}, s)]$.
%
Here $\Phi(\vec{x}, y, t, s)$ represents $ \bigvee_{1\leq i \leq k} (\gamma_i(\vec{x}, s) \land \delta_i(\vec{x}, y, t, s))$, where $\gamma_i$ is a context and $\delta_i$ is a relevant formula to be used to compute the temporal fluent $f$'s value when $\gamma_i$ holds.  $\Psi$ denotes  $\bigvee_{1\leq i \leq k} \gamma_i(\vec{x}, s)$, which represents all the mutually exclusive relevant contexts of a temporal fluent. 
Thus the above SEA states that the value of a temporal fluent $f(\vec{x})$ changes only if some context $\gamma_i$ holds and according to the rules defined in the formula $\delta_i$ associated with $\gamma_i$; otherwise, it remains unchanged. The formula $\delta_i(\vec{x},y,t,s)$ implicitly or explicitly defines $y$ using some arbitrary (domain-specific) constraints on the variables and fluents.
\par
A \emph{hybrid BAT} \cite{BatusovGS19a} 
is 
a collection of axioms 
$\mathcal{D}_{S_0}\cup\mathcal{D}_{ap}\cup\mathcal{D}_{ss}\cup\mathcal{D}_{se}\cup\mathcal{D}_{una}\cup\Sigma$, where $\mathcal{D}_{se}$ is the set of state evolution axioms.
\noindent\paragraph{\textbf{Example}} \label{ExpNPP}
We use a simple nuclear power plant (NPP) 
as a running example, 
where the core temperature needs to be maintained below a certain threshold using 
a cooling system and 
coolant supplied through pipes attached to it. For simplicity, we assume just one (unnamed) pipe per plant. In this domain, we have the following actions: 
    $\mathit{rup}(p,t),$ i.e.\ the pipe in plant $p$ ruptures at time $t$;
    $\mathit{csFailure}(p,t)$, i.e.\ the cooling system of $p$ fails at $t$;
    $\mathit{fixP}(p,t)$, i.e.\ the pipe of plant $p$ is fixed at $t$;
    $\mathit{fixCS}(p,t)$, i.e.\ the cooling system of $p$ is fixed at $t$; and
    $\mathit{mRad}(p,t)$, i.e.\ the radiation level of $p$ is measured at $t$. 
%
The discrete fluents in this domain are: 
    $\mathit{Ruptured(p,s)}$, representing plant $p$ has a ruptured pipe in situation $s$, and
    $\mathit{CSFailed(p,s)}$, representing the cooling system of $p$ has failed in $s$. 
We also have a temporal functional fluent $\mathit{coreTemp(p,t,s)}$, which stands for the core temperature of power plant $p$ at time $t$ in situation $s$. 
\par
We now give the 
action precondition axioms for these actions.
\par
\vspace{-1 em}
\begin{small}
\begin{eqnarray*}
&&\hspace{-7 mm}\mathit{Poss}(\mathit{rup}(p,t), s),\;\;\;\;\;\;\;\;\mathit{Poss}(\mathit{fixP}(p, t), s) \equiv \mathit{Ruptured}(p,s),\\
&&\hspace{-7 mm}\mathit{Poss}(\mathit{csFailure}(p,t),s)\equiv\neg\mathit{CSFailed}(p,s),\\
&&\hspace{-7 mm}\mathit{Poss}(\mathit{fixCS}(p,t),s)\equiv\mathit{CSFailed}(p,s),\;\;\;\;\;\;\;\;\mathit{Poss}(\mathit{mRad}(p,t), s).    
\end{eqnarray*}
\end{small}
\noindent 
We also have the following successor-state axioms:
\par\vspace{-1 em}
\begin{small}
\begin{eqnarray*}
&&\hspace{-7 mm}\mathit{Ruptured}(p,do(a,s))\equiv\exists t.\;a=\mathit{rup}(p,t)\\
&&\hspace{20 mm}\mbox{}\lor (\mathit{Ruptured}(p,s)\land\neg\exists t.\;a=\mathit{fixP}(p,t)),\\
&&\hspace{-7 mm}\mbox{}\mathit{CSFailed}(p,do(a,s))\equiv\exists t.\;a=\mathit{csFailure}(p,t)\\
&&\hspace{20 mm}\mbox{}\lor (\mathit{CSFailed}(p,s)\land\neg\exists t.\;a=\mathit{fixCS}(p,t)).
\end{eqnarray*}
\end{small}
\noindent Thus, e.g., the first axiom says that the pipe of plant $p$ has ruptured after action $a$ happens in situation $s$, i.e.\ in $\mathit{do}(a, s)$, iff $a$ is the action of rupturing the pipe of $p$ at some time $t$, or the pipe of $p$ was already ruptured in $s$ and $a$ does not refer to the action of fixing the pipe of $p$ at some time $t$. 
\par
The contexts for core temperature $\gamma_i$, $i$ = 1 to 3, comprise: 
$\gamma_1(p) = \mathit{Ruptured}(p) \land \mathit{CSFailed}(p)$; 
$\gamma_2(p) = \mathit{Ruptured}(p) \land \neg\mathit{CSFailed}(p)$; 
and 
$\gamma_3(p) = \neg\mathit{Ruptured}(p) \land \mathit{CSFailed}(p)$. 
%
We now give state evolution axiom for temporal fluent $\mathit{coreTemp}(p, t, s)$.
\par\vspace{-1 em}
\begin{small}
\begin{eqnarray*}
&&\hspace{-7 mm}\mbox{}\mathit{coreTemp}(p, t, s) = y\equiv{}\\ 
&&\hspace{-7 mm}[(\gamma_1(p,s)\!\land\!\delta_1(p,t,s))\lor(\gamma_2(p,s)\!\land\!\delta_2(p,t,s))
 \lor(\gamma_3(p,s)\!\land\!\delta_3(p,t,s))\\
&&\hspace{-7 mm}\mbox{}\lor (y = \mathit{coreTemp}(p, start(s), s) \land \neg(\gamma_1(p,s)\vee\gamma_2(p,s)\vee\gamma_3(p,s)))].
\end{eqnarray*} 
\end{small}
\noindent Thus, the value of $\mathit{coreTemp}$ of $p$ at time $t$ in situation $s$ is dictated by 
$\delta_1$ if both $p$'s cooling system has failed and its pipe was ruptured; by
$\delta_2$ if $p$'s pipe was ruptured but its cooling system is working; by
$\delta_3$ if $p$'s cooling system has failed but its pipe is intact; and 
it remains the same as in $start(s)$, otherwise. 
Here, $\delta_i$ for $i=1,2,3$ is defined as follows: 
$\delta_i(p,t,s) \defi{}$
$\mathit{coreTemp}(p, t, s) = \mathit{coreTemp}(p, \mathit{start}(s), s) + (t - \mathit{start}(s)) \times\Delta_i,$ 
where, $\Delta_i$ is the rate of change such that $\Delta_1=100, \Delta_2=35,$ and $\Delta_3=55.$ The above formula computes $\mathit{coreTemp}(p, t, s)$ by adjusting the initial temperature at $\mathit{start}(s)$ based on the elapsed seconds $t - \mathit{start}(s)$ and specifies a rate of temperature increase, 100, 35, or 55 degrees per second, respectively, depending on the context $\gamma_i$. For simplicity, we use these basic equations, but we could have used more realistic differential equations to model temperature change as well.
\par
We assume the availability of unique names axioms for actions. 
Finally, we assume that there is at least one nuclear power plant $P_1$, for which the following initial state axioms hold: 
$\neg\mathit{Ruptured}(P_1, S_0)$; 
$\neg\mathit{CSFailed}(P_1,S_0)$; and 
$\mathit{coreTemp}(P_1,\mathit{start}(S_0),S_0) = -50$.
%
Henceforth, we use $\mathcal{D}_\mathit{npp}$ to refer to the above axiomatization. \hfill\qed
\section{Actual Achievement Cause in the SC}
Given a history of actions/events (often called a scenario) and an observed effect, \emph{actual causation} involves figuring out which of these actions are responsible for bringing about this effect. 
When the effect is assumed to be false before the execution of the actions in the scenario and true afterwards, the notion is referred to as \emph{achievement (actual) causation}. Based on Batusov and Soutchanski's original proposal \cite{BatusovS18}, Khan and Lesp\'{e}rance (KL) recently offered a definition of achievement cause in the SC \cite{KL21}. Both of these frameworks assume that the scenario is a linear sequence of actions, i.e.\ no concurrent actions are allowed. 
%
\par 
To formalize reasoning about effects, KL \cite{KL21} introduced the notion of \emph{dynamic formulae}. 
An effect $\varphi$ in their framework is thus a dynamic formula.\footnote{While KL also study the epistemics of causation and formalized epistemic dynamic formulae, we restrict our discussion to objective causality only. Also, while we provide the intuition behind KL's definition of actual cause, familiarity with its technical formalization is not required for this paper; for completeness, the details are given in the technical appendix.} 
As usual, we will often suppress the situation argument of $\varphi$. 
Given an effect $\varphi,$ the actual causes are defined relative to a {\em narrative} (variously known as a {\em scenario} or a {\em trace}) $s$. When $s$ is ground, the tuple $\langle\varphi,s\rangle$ is often called a {\em causal setting} \cite{BatusovS18}. Also, it is assumed that $s$ is executable, and $\varphi$ was false before the execution of the actions in $s$, but became true afterwards, i.e.\ 
$\mathcal{D}\models \mathit{Exec}(s)\wedge\neg\varphi[S_0]\wedge\varphi[s]$. Here $\varphi[s]$ denotes the formula obtained from $\varphi$ by restoring the appropriate situation argument into all fluents in $\varphi$ (see the appendix for a definition
).
%
%
%
%
\par
Note that since all changes in the SC result from actions, the potential causes of an effect $\varphi$ are identified with a set of action terms occurring in $s$. However, since $s$ might include multiple occurrences of the same action, one also needs to identify the situations where these actions were executed. To deal with this, KL required that each situation be associated with a time-stamp, which is an integer for their theory.  Since in the context of knowledge, there can be different epistemic alternative situations (possible worlds) where an action occurs, using time-stamps provides a common reference/rigid designator for the action occurrence. KL assumed that the initial situation starts at time-stamp 0 and each action increments the time-stamp by one. Thus, their action theory includes the following axioms: 
%
$\mathit{timeStamp}(S_0)=0;$ 
$\forall a,s,ts.\;\mathit{timeStamp}(do(a,s))=ts\equiv \mathit{timeStamp}(s)=ts-1.$
%
\noindent With this, causes in their framework is a non-empty set of action-time-stamp pairs derived from the trace $s$ given $\varphi$. 
%
%
\par
We will now present KL's definition of causes in the SC. The idea behind how causes are computed is as follows. Given an effect $\varphi$ and scenario $s$, if some action of the action sequence in $s$ triggers the formula $\varphi$ to change its truth value from false to true relative to $\mathcal{D}$, and if there are no actions in $s$ after it that change the value of $\varphi$ back to false, then this action is an actual cause of achieving $\varphi$ in $s$. Such causes are referred to as {\em primary} or {\em direct} causes:
\begin{definition}[Primary Cause \cite{KL21}]\label{PCause}\label{PCause0}
\begin{small}
\begin{eqnarray*}
&&\hspace{-7 mm}\mathit{CausesDir}(a,ts,\varphi,s)\defi
\exists s_a.\;\mathit{timeStamp}(s_a)=ts\wedge(S_0<do(a,s_a)\leq s)\\
&&\hspace{3 mm}\mbox{}\wedge\neg\varphi[s_a]\wedge\forall s'.(do(a,s_a)\leq s'\leq s\supset\varphi[s']).
\end{eqnarray*}
\end{small}
\end{definition} 
\noindent That is, $a$ executed at time-stamp $ts$ is the \emph{primary cause} of effect $\varphi$ in situation $s$ 
iff $a$ was executed in a situation with time-stamp $ts$ in scenario $s$, $a$ caused $\varphi$ to change its truth value to true, and no subsequent actions on the way to $s$ falsified $\varphi$. 
\par
Now, note that a (primary) cause $a$ might have been non-execut\-able initially. Also, $a$ might have only brought about the effect conditionally and this context condition might have been false initially. Thus earlier actions on the trace that contributed to the preconditions and the context conditions of a cause must be considered as causes as well. KL introduced $\mathit{Causes}(a,ts,\varphi,s)$ to inductively capture both primary and indirect causes; see appendix.
\noindent\paragraph{\textbf{Example (cont'd).}} 
Assume a SC BAT for a variant of our example domain, $\mathcal{D}_\mathit{npp}^\mathit{SC}$ that only includes the atemporal variants of the last 3 actions, i.e.\ $\mathit{mRad}(p),$ $\mathit{csFailure}(p),$ and $\mathit{fixCS}(p)$, the $\mathit{CSFailed}(p,s)$ fluent, and the associated 
axioms. 
Consider the scenario $\sigma_1\;=\;do([\mathit{mRad}(P_1),\mathit{csFailure}(P_1),\mathit{fixCS}(P_1),$ $\mathit{mRad}(P_1),$ $\mathit{csFailure(P_1)},\mathit{mRad}(P_1)],S_0)$ and the effect $\varphi_1$ $=\mathit{CSFailed}(P_1)$, for power-plant $P_1$. We can show that: 
$\mathcal{D}_\textit{npp}^\mathit{SC}\models\mathit{CausesDir}(\mathit{csFailure}(P_1),$ $4,\varphi_1,\sigma_1).$
\noindent Moreover, we have the following result about indirect causes: 
$\mathcal{D}_\textit{npp}^\mathit{SC}\models\mathit{Causes}(\mathit{csFailure}(P_1),1,\varphi_1,\sigma_1)
\land\mathit{Causes}(\mathit{fixCS}(P_1),$ $2,\varphi_1,\sigma_1)
\land\mathit{Causes}(\mathit{csFailure}(P_1),4,\varphi_1,\sigma_1).$
\noindent Explaining backwards, the second $\mathit{csFailure}(P_1)$ action executed at time-stamp 4 is a cause as it is a direct cause. The $\mathit{fixCS}(P_1)$ action is a cause since had it not for this action, the primary cause would not have been executable (see precondition axiom for $\mathit{csFailure}$). Finally, the first cooling system failure action is required to make $\mathit{fixCS}(P_1)$ executable.\footnote{Since we do not deal with secondary causes in this paper, we will not pursue this discussion further and rather refer the interested reader to \cite{KL21}.}
\hfill\qed
%
%

\section{Primary Actual Cause in the HTSC}
The above definition, which provides some insight on which actions are causes, can also be used to reason about atemporal effects in the HTSC.\footnote{Besides continuous time, we also keep KL's time-stamps to uniquely identify causes. Note that while one might be tempted to use the execution time of an action for this, HTSC do allow two actions to have the same execution time, and thus using execution time to identify action instances will not work.} However it cannot handle temporal effects and continuous change. To deal with this, we next formalize causes of temporal effects. 
%
%
%
%
We start by defining a notion of causal setting in the HTSC.
\begin{definition}[Hybrid Temporal Achievement Causal Setting]\label{HTSCSetting}
A \textit{hybrid temporal achievement causal setting} (henceforth, \textit{hybrid setting}) is a tuple $\langle\mathcal{D},\sigma,\varphi\rangle$, where $\mathcal{D}$ is a hybrid BAT, $\sigma\neq S_0$ is a ground situation term of the form $\mathit{do}([\alpha_1,\ldots, \alpha_n],$ $S_0)$ with non-empty sequence of ground action functions $\alpha_1,\ldots,\alpha_n,$ and $\varphi$ is a situation 
and time-suppressed temporal SC formula that is uniform in $s$ (meaning that it has no occurrences of $\mathit{Poss}$, $\sqsubseteq$, other situation terms besides $s$, and quantifiers over situations) such that: 
\[
\mathcal{D}\models\mathit{Exec}(\sigma)\land\neg\varphi[\mathit{start}(S_0),S_0]\land\neg\varphi[\mathit{time}(\alpha_1),S_0]\land \varphi[\mathit{start}(\sigma),\sigma].
\]    
\end{definition}
\noindent 
To rule out cases where no actions contributed to the effect (this can happen when the effect $\varphi$ becomes true during the initial situation as some context was true initially), the above definition requires $\varphi$ to be false at the beginning and at the end of the initial situation $S_0$.\footnote{As we will see later, this restriction is not strong enough, and in some contexts it is indeed still possible for an effect to become true despite having no contributing actions, e.g., when the context that brought about the effect was true initially.} 
To ensure that $\varphi$ is actually achieved within the scenario, we also require $\varphi$ to be true when observed at the beginning of $\sigma$. In a hybrid domain, in general, one can query the causes of an observed effect at any time-point within a situation $\sigma$, i.e.\ at any time-point in between the start-time and the end-time of $\sigma$, inclusive. To simplify, we assume that the query is posed relative to the starting time of $\sigma$. 
If this is not the case, one can always add a subsequent dummy action, $\mathit{noOp}$ (which has no effect and is always possible to execute, i.e.\ $\forall s.\mathit{Poss}(\mathit{noOp},s)$), and query wrt to 
scenario $\mathit{do}(\mathit{noOp},\sigma)$ instead.
\par
In this paper, 
we deal with effects $\varphi$ that are 
conditions on the values of primitive temporal fluents (e.g.\ $\mathit{coreTemp}(P)>1000$) 
exclusively. In the following, we will write $\varphi[t,s]$ to denote the formula obtained from $\varphi$ by restoring the appropriate situation and time arguments into the fluents in $\varphi$, and thus, for example, $\mathit{coreTemp}(P_1)[5, S_0]$ stands for $\mathit{coreTemp}(P_1, 5, S_0)$. 
\par
We are now ready to give our definition of actual cause of a primitive temporal effect. 
Note that for discrete effects, the primary cause, which is an action $a$ executed at time-point $t$, brings about the effect discretely and immediately after the execution of $a$ at $t$. For the temporal case, however, the effect might be only realized after a while, and one or more irrelevant actions might be executed in between. For example, if one changes the temperature on a thermostat, the desired room temperature will likely be achieved after some time has passed, but in between there can be other irrelevant actions that might be executed, those that have no impact on the value of the room temperature. 
Thus, while defining the primary achievement cause, in addition to actions causing the change in a temporal fluent's value, we need to identify the situation where the
effect was actually achieved within the scenario, i.e.\ the ``achievement situation''.
\par
Recall that in the HTSC, the values of temporal fluents can change only when certain relevant contexts are enabled. Contexts for a temporal fluent, which are (mutually exclusive) discrete fluents, on the other hand, are enabled or disabled due to the execution of actions. Thus, when determining the primary cause of some temporal fluent having a certain value, we first need to identify the last context $\gamma$ that was enabled before the fluent acquired this value, i.e.,\ the context $\gamma$ which was true in the achievement situation $s_\varphi$ of the effect $\varphi$, and then figure out the action $a$ that caused/enabled this context in $s_\varphi$. Since contexts are mutually exclusive (no two contexts can be true at the same time), $\gamma$ must have been the only enabled context in the achievement situation $s_\varphi$, which ensures that the action $a$ is unique. Additionally, $a$ must have been the last action that enabled $\gamma$, and whose contribution brought about the temporal effect under consideration.\footnote{There can certainly be other secondary/indirect causes, but we are only concerned with primary causes in this paper.}
%
\par 
In the following, we give the definition of primary cause relative to a hybrid causal setting $\langle\mathcal{D},\sigma,\varphi\rangle$. In this, the effect $\varphi$ is a constraint on the values of a situation- and time-suppressed primitive temporal fluent $f(\vec{x})$. 
Also, $\gamma_i^f$ refers to the contexts $\delta_i(\vec{x}, y, t, s)$, indexed by $i$, that are associated with the temporal fluent $f$.

\begin{definition}[Primary Temporal Achievement Cause]\label{PTI}
\begin{small}
\begin{eqnarray*}
&&\hspace{-7 mm}\mathit{CausesDir_{temp}^{prim}}(a, ts,\varphi,s)\defi\\
&&\hspace{7 mm}\exists s_\varphi.\;\mathit{AchvSit}(s_\varphi,\varphi,s)
\land \exists i.\;\mathit{CausesDir}(a, ts, \gamma_i^f, s_\varphi).
\end{eqnarray*}
\end{small}
\end{definition}
\noindent That is, an action $a$ executed at time-stamp $ts$ directly causes the effect $\varphi$ in scenario $s$ iff the achievement situation of $\varphi$ in $s$ is $s_\varphi$, and $a$ executed in some earlier situation with time-stamp $ts$ directly caused the active context $\gamma_i^f$ for the temporal fluent $f$ mentioned in $\varphi$ in scenario $s_\varphi$. Here, $\mathit{CausesDir}(a, ts, \gamma_i^f\!, s_\varphi)$ is as defined in Definition \ref{PCause0}. Note that, $\mathit{CausesDir}(a, ts, \gamma_i^f\!, s_\varphi)$ implies that the context $\gamma_i^f$ holds in $s_\varphi$, i.e., $\gamma_i^f[s_\varphi]$, and thus it is indeed the (unique) context that was active in $s_\varphi$. Since contexts are mutually exclusive, we do not need to check whether a subsequent action executed after $a$ made another context true before the achievement of the effect in situation $s_\varphi$.
\par
To define the achievement situation, observe that the effect $\varphi$ must be true at the end of the achievement situation and must remain true in all subsequent situations and times. But since there can be multiple situations in between the achievement situation and the final situation in the scenario, we must also ensure that the achievement situation is the earliest such situation to uniquely identify it. The following definition captures this intuition.
\begin{definition}[Achievement Situation]\label{DefAchvSit}
\begin{small}
\begin{eqnarray*}
%
&&\hspace{-7 mm}\mathit{AchvSit}(s_\varphi,\varphi,s)\defi{}\\
&&\hspace{-3 mm}\mathit{AchvSitAux}(s_\varphi,\varphi,s)\land\neg\exists s''.\;(s''\! < s_\varphi\land\mathit{AchvSitAux}(s'',\varphi,s)),\\
&&\hspace{-7 mm}\mathit{where,}\\
&&\hspace{-7 mm}\mathit{AchvSitAux}(s_\varphi,\varphi,s)\defi\varphi[\mathit{end}(s_\varphi, s), s_\varphi]\\
&&\hspace{5 mm}{}\land\forall s'\!, t.\; s_\varphi < s'\! \leq s \land start(s') \leq t \leq \mathit{end}(s'\!, s) \supset \varphi[t, s'].
\end{eqnarray*}
\end{small}
\end{definition}
\noindent That is, $s_\varphi$ is the achievement situation of the effect $\varphi$ in scenario $s$ iff $\varphi$ holds at the end of the situation $s_\varphi$, and $\varphi$ continues to hold in all subsequent situations $s'$ and time points $t$ between $\mathit{start}(s')$ and the ending time of $s'$ in $s$, $\mathit{end(s',s)}$. Additionally, there must not exist another preceding situation $s''$ before $s_\varphi$ that satisfies these conditions.  This ensures that $s_\varphi$ is the earliest situation in which $\varphi$ is achieved and maintained till the start of the scenario $s$.
\par
Here, we use the function $\mathit{end}(s',s)$ to specify the end time of a situation $s'$ within a given scenario $s$. Since it is not directly possible to talk about the end time of a situation in the HTSC (as the end time does not really exist when the scenario is not known), we use the time of the next action within the scenario to denote this.
\begin{definition}[End Time of a Situation within a Context]
\begin{eqnarray*}
&&  \hspace{-7 mm}\mathit{end}(s',s)\defi
    \begin{cases}      
      \mathit{start}(s') & \!\!\textup{if }s'=s\\
      \mathit{time}(a) & \!\!\textup{if }\exists a.\;do(a,s')\leq s\\
    \end{cases}      
\end{eqnarray*}
\end{definition}
\noindent Thus, the end time of a situation $s'$ in scenario $s$ is the starting time of $s'$ if $s'$ is the last situation in scenario $s$, or the time of the execution of the next action $a$ in $s'$ within scenario $s$, i.e., $\mathit{time}(a)$ such that $\mathit{do}(a, s')\leq s$. Note that, since our definition of hybrid setting guarantees that causes are computed relative to the starting time of the scenario, taking the starting time of $s'$ as the end time of it is reasonable when $s' = s$.
\paragraph{\textbf{Example (cont'd).}}\label{PrimExp}
Consider 
setting $\langle \mathcal{D}_\textit{npp}, \varphi_2, \sigma_2 \rangle$, where 
$\varphi_2\defi\mathit{coreTemp}(P_1) \geq 1000$, 
and $\sigma_2 = do([\mathit{rup}(P_1, 5),\mathit{csFailure}(P_1, 15),$ $\mathit{mRad}(P_1, 20),\mathit{fixP}(P_1, 26)],S_0)$. 
Recall that, the last argument of each action represents the execution time of that action. This 
is depicted in Figure \ref{fig:def2}, which also shows the temperature at the beginning and 
end of each situation for clarity. Here, time progresses vertically, while actions unfold horizontally. Also, $S_1=do(\mathit{rup}(P_1,5),S_0)$, 
etc. 
\begin{figure}[t]
\centering
\resizebox{0.42\textwidth}{!}{%
\begin{tikzpicture}[font=\fontsize{10}{13}\selectfont]
\draw[thick] 
    (0,2) node[anchor=east] {$-50^\circ, S_0$} -- (0,0) node[anchor=east] {$-50^\circ, S_0$}
    -- (3,0) node[anchor=south, yshift=-0.5mm, xshift=-16mm] {$\quad\mathit{rup}(P_1, 5)\;\;$} node[anchor=west] {$\varphi_2 = -50^\circ, \gamma_2, S_1$}
    -- (3,-2) node[anchor=east]  {$300^\circ, \gamma_2, S_1$}
    -- (6,-2) node[anchor=south, yshift=-0.5mm, xshift=-15mm] {$\mathit{csFailure}(P_1, 15)$} node[anchor=west] {$300^\circ, \gamma_1, S_2$}
    -- (6,-4) node[anchor=east]  {$800^\circ, \gamma_1, S_2$}
    -- (9.5,-4) node[anchor=south, yshift=-0.5mm, xshift=-17mm] {$\mathit{mRad}(P_1, 20)\;$} node[anchor=west] {$800^\circ, \gamma_1, S_3$}
    -- (9.5,-4.8) node[anchor=west] {$1000^\circ, \varphi_2$}
    -- (9.5,-6) node[anchor=east]  {$1400^\circ, \gamma_1, S_3$}
    -- (12.5,-6) node[anchor=south, yshift=-0.5mm, xshift=-13mm] {$\mathit{fixP}(P_1, 26)\quad$} node[anchor=west] {$1400^\circ, \gamma_3, \sigma_2$};

\filldraw [black] (0,0) circle (2pt);
\filldraw [black] (0,2) circle (2pt);
\filldraw [black] (3,0) circle (2pt);
\filldraw [black] (3,-2) circle (2pt);
\filldraw [black] (6,-2) circle (2pt);
\filldraw [black] (6,-4) circle (2pt);
\filldraw [black] (9.5,-4) circle (2pt);
\filldraw [black] (9.5,-4.8) circle (2pt);
\filldraw [black] (9.5,-6) circle (2pt);
\filldraw [black] (12.5,-6) circle (2pt);
\end{tikzpicture}
}
    \caption{Primary Cause in Hybrid Domains: Primitive Case}
    \label{fig:def2}
\end{figure}
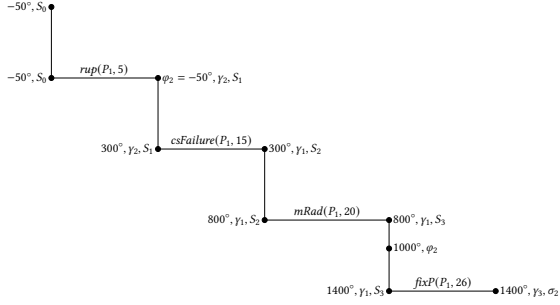
We can show that:
\begin{proposition}\label{Exp1-PDCause}
\begin{small}
\[\mathcal{D}_\textit{npp} \models \mathit{CausesDir}(\mathit{csFailure}(P_1, 15), 1, \gamma_1, S_3).\]
\end{small}
\end{proposition}
\begin{proposition}\label{Exp1-PCause}
\begin{small}
\[\mathcal{D}_\textit{npp} \models \mathit{CausesDir}^\mathit{prim}_\mathit{temp}(\mathit{csFailure}(P_1, 15), 1, \varphi_2, \sigma_2).\]
\end{small}
\end{proposition}
\noindent 
Note that, although $\mathit{mRad}(P_1, 20)$ is the latest action before achieving $\varphi_2$ in situation $S_3$, it is irrelevant because it did not enable any context and thus did not contribute to the change in $\mathit{coreTemp}$. Our definition correctly identifies $\mathit{mRad}(P_1, 20)$ as an irrelevant action. Also, even though $\mathit{rup}(P_1, 5)$ contributed to the change, it is not a primary cause of $\varphi_2$ because it is not the primary cause of the enabled context $\gamma_1$ in achievement situation $S_3$. \qed
\section{Properties}\label{HTSCPropoerties}
We now prove some intuitive properties of our formalization (see the appendix for proof sketches). The first three are self-explanatory. 
\begin{lemma}[Uniqueness of Direct Cause]\label{UniqDirCause}
Given a causal setting $\langle \mathcal{D}, \varphi, s \rangle$, it follows that:
\par\vspace{-3 mm}
\begin{small}
\begin{eqnarray*}
&&\hspace{-5 mm}\mathcal{D} \models \forall a, a'\!, ts, ts'\!.\; \mathit{CausesDir}(a,ts,\varphi,s) \land \mathit{CausesDir}(a',ts',\varphi,s) \\
&& \hspace{10 mm}\supset a = a' \land ts = ts'.
\end{eqnarray*}
\end{small}
\end{lemma}
%
\begin{lemma}[Uniqueness of Achievement Situation]\label{UniqAchvSit}
Given a hybrid setting $\langle \mathcal{D}, \varphi, s \rangle$, we have:
\par\vspace{-3 mm}
\begin{small}
\begin{eqnarray*}
&&\mathcal{D} \models \mathit{AchvSit}(s_\varphi, \varphi, s) \land \mathit{AchvSit}(s'_\varphi, \varphi, s) \supset s_\varphi = s'_\varphi.
\end{eqnarray*}
\end{small}
\end{lemma}
%
\begin{theorem}[Uniqueness of Primary Cause of Temporal Effects]\label{Property-Uniqueness}
Given a hybrid setting $\langle \mathcal{D}, \varphi, s \rangle$, we have:
\par\vspace{-3 mm}
\begin{small}
\begin{eqnarray*}
&&\hspace{-5 mm}\mathcal{D} \models \mathit{CausesDir}^\mathit{prim}_\mathit{temp}(a_1, ts_1, \varphi, \sigma) \land \mathit{CausesDir}^\mathit{prim}_\mathit{temp}(a_2, ts_2, \varphi, \sigma) \\
&& \hspace{10 mm} \supset a_1 = a_2 \land ts_1 = ts_2.
\end{eqnarray*}
\end{small}
\end{theorem}
%
%
%
%
\par
Next, as mentioned above, the primary causes of primitive temporal fluents might not exist (as it may be implicit in 
$S_0$). This holds even if the causal setting under consideration is a proper hybrid setting (as specified by Definition \ref{HTSCSetting}).
%
\begin{theorem}[Implicit Primary Cause]\label{NoCause}
Assume that $\varphi$ is a constraint on the values of a primitive temporal fluent $f$. Then: 
\par\vspace{-3 mm}
\begin{small}
\begin{eqnarray*}
&& \hspace{-7 mm} \mathcal{D} \models (\mathit{ProperHTSCAchvCausalSetting}(\varphi, \sigma)\land \exists s_\varphi.\; \mathit{AchvSit}(s_\varphi, \varphi, \sigma) \\
&& \hspace{0 mm} \mbox{}\land \exists i.\; \gamma_i^f[s_\varphi] \land (\forall s'.\; S_0 \leq s' \leq s_\varphi \supset \gamma_i^f[s']) \\
&& \hspace{0 mm} \supset \neg \exists a, ts.\; \mathit{CausesDir}^\mathit{prim}_\mathit{temp}(a, ts, \varphi, \sigma)),\\
&& \hspace{-7 mm} \text{where,}\\ 
&& \hspace{-7 mm}\mathit{ProperHTSCAchvCausalSetting}(\varphi, \sigma) \defi \mathit{Exec}(\sigma)
\land \exists a_0.\; do(a_0, S_0) \leq \sigma\\
&& \hspace{0 mm}{}\land \neg \varphi[\mathit{start}(S_0), S_0] \land \neg \varphi[\mathit{time}(a_0), S_0] \land \varphi[\mathit{start}(\sigma), \sigma].
\end{eqnarray*}
\end{small}
\end{theorem}
\noindent This states that in a given proper hybrid causal setting $\langle \mathcal{D}, \sigma, \varphi \rangle$, if there exists a context $\gamma_i^f$ such that $\gamma_i^f$ was true in the initial situation $S_0$ and remained true until the effect $\varphi$ was achieved in the achievement situation $s_\varphi$ in $\sigma$, then the primary cause of temporal effect $\varphi$ in $\sigma$ simply does not exist.
%
%
%
\par
Finally, we study the conditions under which primary achievement causes persist when the scenario changes. To this end, we first show a result about the persistence of achievement situations.
\begin{lemma}\label{PersistenceOfAchvSit} Given a hybrid setting $\langle \mathcal{D}, \varphi, s \rangle$, we have:
\begin{small}
\begin{eqnarray*}
&&\hspace{-7 mm}\mathcal{D} \models \mathit{AchvSit}(s_\varphi, \varphi, s) \land s < s^*\\
&&\hspace{-5 mm} \mbox{}\land (\forall s', t.\; s \leq s' \leq s^* \land \mathit{start}(s') \leq t \leq \mathit{end}(s', s^*) \supset \varphi[t, s'])\\
&&\hspace{0 mm} \mbox{} \supset \mathit{AchvSit}(s_\varphi, \varphi, s^*).
\end{eqnarray*}
\end{small}
\end{lemma}
\noindent 
Using this, we can show the following result.
\begin{theorem}[Persistence]\label{PCausePersistence}Given 
setting $\langle \mathcal{D}, \varphi, s \rangle$, we have:
\begin{small}
\begin{eqnarray*} 
&&\hspace{-7 mm}\mathcal{D} \models \mathit{CausesDirectly^{prim}_{temp}}(a, ts, \varphi, s) \\
&&\hspace{-5 mm}\mbox{}\land (\forall s', t'\!.\; s \leq s' \leq s^* \land \mathit{start}(s') \leq t' \leq \mathit{end}(s'\!, s^\ast) \supset \varphi[t'\!, s']) \\
&&\hspace{0 mm}\supset \mathit{CausesDirectly^{prim}_{temp}}(a, ts, \varphi, s^*).
\end{eqnarray*}
\end{small}
\end{theorem}
\noindent Thus, if an action $a$ executed at time-stamp $ts$ is the primary cause of a temporal effect $\varphi$ in scenario $s$, then $a$ remains the primary cause of $\varphi$ in all subsequent situations/scenarios $s^*$ as long as $\varphi$ does not change after it is achieved in $s$. Note that this holds even if the context changes and the value of the associated fluent $f$ in $\varphi$ varies, provided that $\varphi$ itself remains unchanged.
%
%
%
%
%
%
%
%
%

\section{Primary Cause via Contribution}\label{Contrib}
Our definition of actual cause of temporal effects appeals to intuition. However, it neither formalizes an actual cause in counterfactual terms \cite{hume1748, HpCfSc}, nor does it relate to the regularity accounts (e.g.\ \cite{Mackie65}), the two most popular formalizations of causation. As a consequence, it is not clear how this definition compares with these common but different approaches to actual causes. To deal with this, in this section we give a second definition of actual cause of temporal effects using the notion of contributions. This new definition is more similar to those for discrete domains \cite{BeckersV18,KhanS20}, which lets us conceptually relate the two settings. In the next section, we show that this definition, which is equivalent to the previous one, can be justified using a modified but-for test, thus linking it to the counterfactual accounts of causation.\footnote{We could have also used the previous definition to show this.}
%
\par
We now introduce various notions of contributors and define actual cause through contributions in the achievement situation. First, we define \emph{direct possible contributors}, i.e.\  actions that directly initiate the change in the values of temporal fluents.
\begin{definition}[Direct Possible Contributor]\label{CF-DirPossContr}\
Given a hybrid BAT $\mathcal{D}$ and an effect $\varphi$, which is a constraint on the values of a temporal fluent $f$, an action $\alpha$ executed in situation $s_\alpha$ is called a \emph{direct possible contributor to} $\varphi$, denoted $\mathit{DirPossContr}(\alpha, s_\alpha, \varphi)$, iff: 
\par\vspace{-3 mm}
\begin{small}
\begin{eqnarray*}
&&\hspace{-8 mm}\mathcal{D}\models\exists i,s_\varphi,\sigma,ts.\;\mathit{Exec}(s_\alpha)\land\mathit{Poss}(\alpha,s_\alpha)\land \mathit{timeStamp}(s_\alpha)=ts\\ 
&&{}\land s_\alpha<s_\varphi\leq\sigma\land \neg\varphi[time(\alpha),s_\alpha]\land\varphi[end(s_\varphi,\sigma),s_\varphi]\\
&&\hspace{0 mm}{}\land\mathit{CausesDir}(\alpha, ts, \gamma_i^f, s_\varphi).
\end{eqnarray*}
\end{small}
%
\end{definition}
\noindent That is, $\alpha$ executed in situation $s_\alpha$ is a direct possible contributor of $\varphi$, i.e.\ $\mathit{DirPossContr}(\alpha, s_\alpha, \varphi)$, iff $s_\alpha$ is an executable situation, it is possible to execute $\alpha$ in $s_\alpha$, the timestamp of $s_\alpha$ is $ts$, the effect was false in $s_\alpha$ but later became true by the end of a future situation $s_\varphi$ within some scenario $\sigma$, and $\alpha$ executed in timestamp $ts$ is a direct cause of some context $\gamma^f_i$ of $f$ in situation $s_\varphi$. Note that, the last conjunct, i.e.\ $\mathit{CausesDir}(\alpha, ts, \gamma_i^f, s_\varphi)$, ensures that $\gamma^f_i$ has been true since the execution of $\alpha$, and thus given that $\neg \varphi[time(\alpha), s_\alpha] \land \varphi[end(s_\varphi, \sigma), s_\varphi]$, it must have been the enabled context when $\varphi$ was achieved. Thus $\alpha$ must have been the action that achieved $\varphi$. It should also be noted that $\varphi$ and $\gamma_i^f$ are not required to be true in $\sigma$. Finally, note that $s_\varphi$ and $\sigma$ are not guaranteed to be within the actual scenario. 
In what follows, we will use a variant of $\mathit{DirPossContr}$ that makes $s_\varphi$ and $\sigma$ explicit, i.e.\ $\mathit{DirPossContr}(\alpha, s_\alpha, s_\varphi, \sigma, \varphi)$.
\par
Next, we define \emph{direct actual contributors}, which are direct possible contributors that are contained within a given setting/scenario.
\begin{definition}[Direct Actual Contributor]\label{CF-DirActContr}
Given a hybrid setting $\langle\mathcal{D}, \sigma, \varphi\rangle$, an action $\alpha$ executed in situation $s_\alpha$ is called a direct actual contributor to an effect $\varphi$ (which is a constraint on the values of a temporal fluent $f$), i.e.\ $\mathit{DirActContr}(\alpha, s_\alpha, s_\varphi, \varphi, \sigma)$, iff:
\begin{eqnarray*}
&&\hspace{0 mm}\mbox{} \mathcal{D} \models \exists\sigma'\!.\; \mathit{DirPossContr}(\alpha, s_\alpha, s_\varphi, \sigma', \varphi) \land \sigma' \leq \sigma.
\end{eqnarray*}
\end{definition}
\noindent That is, $\alpha$ executed in $s_\alpha$ is a direct actual contributor to $\varphi$ which was brought about in situation $s_\varphi$ within scenario $\sigma$ iff $\alpha$ executed in $s_\alpha$ is a direct possible contributor of $\varphi$ for some scenario $\sigma'$, $s_\varphi$ is an achievement situation of $\varphi$, and $\sigma'$ occurs within the scenario $\sigma$.
\par
We define primary cause as a direct actual contributor such that, after its contribution, the effect is achieved and persists until the end of a scenario—i.e.\ a direct actual contributor to an active context in the achievement situation. We now give this second definition.
\begin{definition}[Primary Temporal Achievement Cause via Contribution]\label{CF-PrimaryCause}
Given a hybrid setting $\langle\mathcal{D},\sigma,\varphi\rangle$, an action $\alpha$ executed at time-stamp $ts$ is called the primary cause of the effect $\varphi$, i.e.\ $\mathit{PrimCause}(\alpha, ts, \varphi, \sigma)$, iff:
\begin{eqnarray*}
&&\hspace{-7 mm}\mbox{}\mathcal{D} \models \exists s_\alpha, s_\varphi.\;  \mathit{AchvSit}(s_\varphi,\varphi,\sigma) \land \mathit{timeStamp}(s_\alpha) = ts\\
&&\hspace{9 mm}{}\land\mathit{DirActContr}(\alpha, s_\alpha, s_\varphi, \varphi, \sigma).
\end{eqnarray*}
\end{definition}
\noindent Thus, $\alpha$ executed at $ts$ is the primary cause of $\varphi$ in $\sigma$ iff $s_\varphi$ is the achievement situation of $\varphi$ in $\sigma$, and $\alpha$ executed in some situation $s_\alpha$ with timestamp $ts$ is a direct actual contributor of $\varphi$ in $\sigma$.
\par
It is not very difficult to see that this new definition is equivalent to the previous one in Definition \ref{PTI}. Formally: 
\begin{theorem}[Equivalence of Primary Cause and Direct Cause Definitions]\label{thm:equivalence}\
Given a hybrid setting $\langle\mathcal{D}, \sigma, \varphi\rangle$, we have:
\begin{eqnarray*}
&&\hspace{-7 mm}\mathcal{D}\models\forall\alpha_1,\alpha_2,ts_1,ts_2.\;(\mathit{CausesDir_{temp}^{prim}}(\alpha_1, ts_1, \varphi, s)\\ 
&&\hspace{10 mm}{}\land\mathit{PrimCause}(\alpha_2,ts_2,\varphi,\sigma))
\supset(\alpha_1 = \alpha_2 \land ts_1 = ts_2).
\end{eqnarray*}
\end{theorem}
%
%
%
%

\section{A Modified But-For Test}\label{ButFor}
In our attempt to intuitively justify the proposed definitions, we now formalize a modified but-for test for validating actual cause. Recall that the but-for test fails due to the presence of preempted actions. Namely, even if one were to remove the actual cause from the scenario, the effect might still follow due to a subsequent action in the scenario whose effect in the original scenario was preempted by the actual cause. Our modification handles this issue by removing all such preempted actions before testing for the effect. 
\par
A particular complication that arises in the context of hybrid domains is that preempted causes can even occur \emph{before} the direct cause\footnote{Note that this cannot be the case in discrete domains; otherwise, by definition, the preempted cause would have been the actual cause.}. To address this problem, we will identify the preempted causes/contributors. We will then introduce the concept of a ``defused'' situation, where the actual cause along with all preempted contributors are replaced with $\mathit{noOp}$ actions, i.e.\ actions that are always possible and that have no effect (the domain modeler must ensure this). This substitution allows us to isolate the impact of the primary cause by removing the influence of the cause and its preempted contributors. We then show that in this defused situation, either the effect does not hold or the scenario becomes non-executable, as formalized in Theorem \ref{CF-Dep-Therorem0}.
\par
To formalize this, we need to define counterfactual situations, which are worlds that would have been realized had actions/events been different from what actually occurred. 
%
For this paper, our notion of counterfactual situations assumes that 
a single action in the given situation is replaced with a different action to produce counterfactual situations. 
\begin{definition}[Single-Action Counterfactual Situation]\label{CF-SingleAction}
\begin{small}
\begin{eqnarray*}
&&\hspace{-7 mm}\mathit{CF}_\mathit{one}(s',s)\defi
%
\exists a_1,a_2,s_\mathit{sh}.\;a_1\neq a_2\land do(a_1,s_\mathit{sh})\sqsubseteq s\land do(a_2,s_\mathit{sh})\sqsubseteq s' \\
&&\hspace{1 mm}\mbox{}\land\forall a^*,s^*.\;do(a_1,s_\mathit{sh})\sqsubset do(a^*,s^*)\sqsubseteq s\\
&&\hspace{8 mm} \supset(\exists s^+.\;\mathit{timeStamp}(s^*)=\mathit{timeStamp}(s^+)
\land 
do(a^*,s^+)\sqsubseteq s')\\
&&\hspace{1 mm}\mbox{}\land\forall a^*,s^*.\;do(a_2,s_\mathit{sh})\sqsubset do(a^*,s^*)\sqsubseteq s'\\ 
&&\hspace{8 mm}\mbox{}
\supset(\exists s^+.\;\mathit{timeStamp}(s^+)=\mathit{timeStamp}(s^*)
\land 
do(a^*,s^+)\sqsubseteq s).
\end{eqnarray*}
\end{small}
\end{definition}
\noindent That is, given a situation $s$, another situation $s'$ is counterfactual to $s$ and differs from $s$ by just one action iff $s$ and $s'$ share a common situation $s_\mathit{sh}$ in their history, the actions performed in the history of $s$ and $s'$ in the situation $s_\mathit{sh}$ are different, but all other actions in their history (performed before and after $s_\mathit{sh}$) are exactly the same. Here, we use the function $\mathit{timeStamp}(s)$ to ensure that the subsequent actions after the unmatched one are performed in exactly the same order in both histories. Note that, since $\mathit{s_{sh}}$ is a common situation in the history of both $s$ and $s'$, it trivially follows that all actions performed in the history of these scenarios until $s_{sh}$ are identical.
\par
We also define an executable variant that ensures that the counterfactual situation obtained is executable. For this, we use a variant of $\mathit{CF}_\mathit{one}$, $\mathit{CF}_\mathit{one} (s'\!,s,\langle a',a,\mathit{ts}\rangle)$, that makes $a = a_1, a' = a_2,$ and $\mathit{ts}=\mathit{timeStamp}(s_{sh})$ explicit (these are stored as a triple).
\begin{definition}[Executable $\mathit{CF}_\mathit{one}$]
\begin{small}
\begin{eqnarray*}
&&\hspace{-7 mm}\mathit{CFEx}_\mathit{one}(s',s,\langle a',a,\mathit{ts}\rangle)\defi
\mathit{CF}_\mathit{one}(s',s,\langle a',a,\mathit{ts}\rangle)\land\mathit{Exec}(s').
%
\end{eqnarray*}
\end{small}
\end{definition}
%
\par
With this, we can show that as expected, the but-for test fails.
\begin{theorem}\label{DisCF}[Preempted Contribution]
    \begin{small}
    \begin{eqnarray*}
        &&\hspace{-7 mm}\mathcal{D}\nvDash \mathit{Causes}(a,\mathit{ts},\phi,s)\supset
        \neg\exists s'.\;\mathit{CFEx}_\mathit{one}(s',s,\langle\mathit{noOp},a,ts\rangle)\vee\neg\phi(s').
    \end{eqnarray*}
    \end{small}
\end{theorem}
\noindent Thus, it is not guaranteed that if $a$ executed in timeStamp $ts$ is a cause of $\varphi$ in scenario $s$, then either an executable counterfactual situation to $s$ obtained by replacing $a$ at $ts$ by $\mathit{noOp}$ does not exist, or the effect $\varphi$ can no longer be observed in such a counterfactual scenario $s'$. This indicates that removing the cause will not necessarily make the effect $\varphi$ disappear or render the scenario non-executable, as $\varphi$ might still follow due to preempted contributors occurring later in the scenario, i.e., actions that would have brought about the effect in the original scenario $s$ had it not for the actual cause $a$. 
\par
A similar result can be shown for the hybrid case. 
Moreover, as mentioned above, in hybrid domains preempted causes can occur even before the direct cause. This is because the removal of the actual cause (and the context realized by it) might result in another context brought about by an earlier action persisting, which might bring about the effect eventually. To address this issue, we use the following approach: after replacing the primary cause with a $\mathit{noOp}$ action with an appropriate time argument, we check if another action becomes the primary cause in the new scenario. If a new primary cause emerges, it must have been a preempted action. So we replace this new primary cause with another $\mathit{noOp}$ action and repeat this process until no primary causes remain. By doing this, we thus not only remove the primary cause from the scenario but also all preempted causes/actions. We will then show in Theorem \ref{CF-Dep-Therorem0} that in the resultant defused scenario, either the effect no longer holds, or the scenario itself becomes non-executable, unless some context inherent in the initial situation brings about the effect.
\par
We now define the preempted contributors/actions.
\begin{definition}[Preempted Contributors]\label{CF-PreemptedContributors}
\sloppy Given a hybrid setting, $\langle\mathcal{D},\sigma,\varphi\rangle$, an action $a$, time-stamp $ts$, and situation $\sigma'\!$, $\mathit{PreempContr}$ $(a ,ts, \sigma'\!, \varphi, \sigma)$ is defined as follows:
\par\vspace{-3 mm}
\begin{small}
\begin{eqnarray*}
&&\hspace{-7 mm}\mathit{PreempContr}(a ,ts, \sigma'\!, \varphi, \sigma)\defi{}\\ 
&&\hspace{-7 mm}\forall P.[\forall a, ts, \varphi, \sigma, \sigma'\!. (\mathit{PrimCause}(a, ts, \varphi, \sigma)\land{} \\
&&\hspace{2 mm}\mathit{CF}_\mathit{one}(\sigma'\!,\sigma,\langle\mathit{noOp(time(a))},a,ts\rangle) \supset P(a,ts, \sigma'\!, \varphi, \sigma))\\
&&\hspace{-3 mm}\mbox{}\wedge\forall a'\!,ts'\!, \varphi, \sigma, \sigma'\!.\;( \exists a''\!,ts''\!,\sigma''\!.(P(a''\!, ts''\!, \sigma''\!, \varphi, \sigma)\land{}\\
&&\hspace{2 mm}\mathit{PrimCause}(a', ts'\!, \varphi, \sigma'')\!\land\!\mathit{CF}_\mathit{one}(\sigma'\!,\sigma''\!,\langle\mathit{noOp(time(a'))},a'\!,ts'\rangle)\\
&&\hspace{43 mm}\mbox{}\supset P(a',ts', \sigma'\!, \varphi, \sigma))\\
&&\hspace{-3 mm}]\supset P(a,ts, \sigma'\!, \varphi, \sigma).
\end{eqnarray*}
\end{small}
\end{definition}

\noindent Thus, $\mathit{PreempContr}$ is defined to be the least relation $P$ such that if $a$ executed at time-stamp $ts$ is a primary cause of $\varphi$ in scenario $\sigma$, then $(a,ts,\sigma'\!, \varphi, \sigma)$ is in $P$, where $\sigma'$ is a single action counterfactual situation to $\sigma$ obtained by replacing $a$ at $ts$ with $\mathit{noOp(time(a))}$; and if $a''$, $ts''$, $\sigma''$, $\varphi$, and $\sigma$ is in $P$, $a'$ executed at time-step $ts'$ is a primary cause of $\varphi$ in $\sigma''$, and $\sigma'$ is a single action counterfactual situation of $\sigma''$ that can be obtained by replacing $a'$ at $ts'$ with the $\mathit{noOp(time(a'))}$ action, then $(a'\!,ts'\!,\sigma'\!,\varphi, \sigma)$ is also in $P$. 
\par
Here, while we do not check the executability of the updated scenario (e.g. by using $\mathit{CFEx_{one}}$ instead of $\mathit{CF_{one}}$), this is guaranteed by our definition of primary cause, which requires the scenario to be executable (see Definition \ref{CF-PrimaryCause}). Note that $\mathit{PreempContr}$ returns a set of tuples containing $\sigma'$, where $\sigma'$ is a situation obtained by removing the actual cause and zero or more preempted contributors. We need to identify the tuple containing a situation $\sigma'$ where \emph{all} the preempted contributors are replaced with $\mathit{noOp}$ actions.\footnote{In fact we might not be able to remove all preempted contributors, e.g., when the scenario becomes non-executable. So to be precise, we would like to identify the tuple where the maximum number of preempted causes are removed.} That is, we are looking for the tuple with the highest number of $\mathit{noOp}$ actions. To achieve this, let us define the number of $\mathit{noOp}$ actions in $\sigma$, denoted by $|\sigma|$.
\begin{definition}[$|\sigma|$]\label{NumberOfNoOpActions}
\begin{small}
\begin{eqnarray*}
&&  \hspace{-7 mm}|s| =
    \begin{cases}      
      0, \quad\quad\;\quad\mathit{if}\quad s = S_0,\\
      0 + |s'|, \quad\mathit{if}\quad s = do(a, s') \land \neg\exists t.\; a = \mathit{noOp}(t),\\
      1 + |s'|, \quad\mathit{if}\quad \exists t.\; s = do(\mathit{noOp}(t), s').
    \end{cases}      
\end{eqnarray*}
\end{small}
\end{definition}
\noindent Using this and $\mathit{PreempContr}$, we define a defused situation as one that contains the maximum number of $\mathit{noOp}$ actions.
\begin{definition}[Defused Situation relative to $\varphi$ and $\sigma$]\label{CF-MaxPreempContrCFSit}
\begin{small}
\begin{eqnarray*}
&&\hspace{-7 mm} \mathit{DefusedSit}(\varphi, \sigma, \sigma') \defi 
\exists a'\!, ts'\!.\;\mathit{PreempContr}(a'\!, ts'\!, \sigma',\varphi,\sigma)\land{}\\
&&\hspace{-7 mm}\forall \sigma''\!\!, a''\!\!, ts''\!\!.\;\mathit{PreempContr}(a''\!, ts''\!, \sigma'',\varphi,\sigma)\land\sigma'\neq\sigma''\supset |\sigma''|<|\sigma'|.
\end{eqnarray*}
\end{small}
\end{definition}
\noindent That is, $\sigma'$ is the defused situation of $\sigma$ with respect to $\varphi$ if there exists an action $a'$ and a time-stamp $ts'$ such that $( a'\!, ts'\!, \sigma'\!, \varphi, \sigma)$ is in the set of preempted contributors, and for any other $\sigma''$ such that $( a''\!, ts''\!, \sigma''\!, \varphi, \sigma)$ is in the set of preempted contributors for some $a''\!$ and $ts''\!$, if $\sigma'\!$ is different than $\sigma''\!$, then $\sigma'$ must have more $\mathit{noOp}$ actions in it. Thus a defused situation is one where the actual cause along with the most number of preempted causes are removed. Note that this might involve removing all the preempted causes. It is also possible that not all preempted contributors were removed, e.g., when replacing an action with $\mathit{noOp}$ made the scenario non-executable (and thus the preempted action is no longer an achievement cause in the modified scenario).
\par
The following are some properties of defused situations.
\begin{lemma}
\begin{small}
\begin{eqnarray*}
&&\hspace{-7 mm}\mathcal{D}\models\exists a, ts.\; \mathit{PrimCause}(a, ts, \varphi, \sigma) \supset \exists \sigma'\!.\; \mathit{DefusedSit}(\varphi, \sigma, \sigma'),\\
&&\hspace{-7 mm}\mathcal{D}\models\exists a_1, ts_1, \sigma_1, \varphi, \sigma, a_2, ts_2, \sigma_2, \varphi, \sigma.\; \mathit{PreempContr}(a_1, ts_1, \sigma_1, \varphi, \sigma) \\
&&\hspace{0 mm}\mbox{}
\land \mathit{PreempContr}(a_2, ts_2, \sigma_2, \varphi, \sigma) \land |\sigma_1| = |\sigma_2| \supset \sigma_1 = \sigma_2,\\
&&\hspace{-7 mm}\mathcal{D}\models\exists \sigma'\!, \sigma''\!.\; \mathit{DefusedSit}(\varphi, \sigma, \sigma') \land \mathit{DefusedSit}(\varphi, \sigma, \sigma'') \supset \sigma' = \sigma'',\\
&&\hspace{-7 mm}\mathcal{D}\models\mathit{DefusedSit}(\varphi, s, s') \supset \neg\exists b, ts_b. \mathit{PrimCause}(b, ts_b, \varphi, s').
\end{eqnarray*}
\end{small}
\end{lemma}
\par
Finally, we present the modified but-for test, which shows that the effect is indeed counterfactually dependent on the primary cause in the sense that if the cause along with (most of) the preempted actions are removed to obtain a defused situation, either the effect does not follow in this situation, or the scenario becomes non-executable, unless a relevant context was already true in 
$S_0$. 
\begin{theorem}[Counterfactual Dependence]\label{CF-Dep-Therorem}\label{CF-Dep-Therorem0}
Given a hybrid setting, $\langle\mathcal{D},\sigma,\varphi\rangle$, where $\varphi$ is a constraint on the values of $f$, we have:
\par\vspace{-3 mm}
\begin{small}
\begin{eqnarray*}
&&\hspace{-7 mm} \exists a, ts.\;\mathit{PrimCause}(a, ts, \varphi, s) \supset{} 
\\
&&\hspace{-7 mm}[\exists s'\!.\;\!\mathit{DefusedSit}(\varphi, s, s')\!\land\!(\wedge_{i} \neg\gamma_i^f[S_0] \supset    \neg(\varphi[start(s'), s']\!\land\!\mathit{Exec}(s')))].
\end{eqnarray*}
\end{small}
\end{theorem}
\noindent This states that if there is an action $a$ and timestamp $ts$ such that $a$ executed at $ts$ is a primary cause of effect $\varphi$ in scenario $s$, then there is a defused situation $s'$ relative to $\varphi$ and $s$, and if additionally it is known that all the contexts of $\varphi$ are inactive initially in $S_0$ (i.e., $\wedge_{i} \neg\gamma_i^f[S_0]$), then it must be the case that $\varphi$ is false in the defused situation $s'$, unless $s'$ has become non-executable.
%
%
\par
Note that, the above analysis is not meant to be a proof of correctness. For instance, if one were to define actual cause as the last non-$\mathit{noOp}$ action in the scenario, the property would still follow. Instead, the point of the above theorem is to show an intuitively justifiable property of causes, that if one removes causes and preempted actions from the scenario, under certain reasonable and intuitive conditions (i.e., that no context of the temporal fluent holds initially), the effect will disappear.
\paragraph{\textbf{Example (cont'd).}}
Returning to our example, we can show that:
\begin{proposition}
\begin{eqnarray*} 
&&\hspace{-7 mm}\mathcal{D}_\textit{npp} \models\mathit{DefusedSit}(\varphi_2,\sigma_2,\sigma_2')\land\mathit{Exec}(\sigma_2')\land\neg\varphi_2[start(\sigma_2'), \sigma_2'],\\
&&\hspace{-7 mm}\mathit{where,}\\
&&\hspace{-5 mm}\sigma_2'=do([\mathit{rup}(P_1, 5),\mathit{noOp}(15),\mathit{mRad}(P_1, 20),\mathit{fixP}(P_1, 26)],S_0).
\end{eqnarray*}
\end{proposition}
\noindent Thus the defused situation of $\sigma_2$ relative to $\varphi_2 = \mathit{coreTemp}(P_1) \geq 1000$ is $\sigma_2'$, which is depicted in Figure \ref{fig:cf1}. Also, $\sigma_2'$ is executable, but the effect $\varphi_2$ does not hold in $\sigma_2'$. \qed
%
%
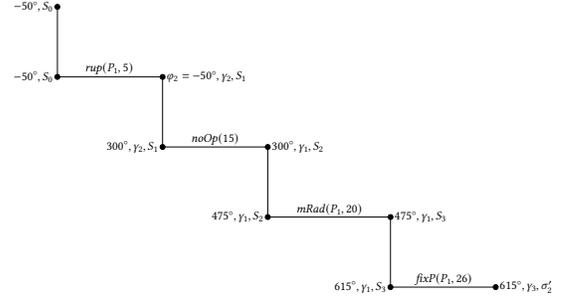
\begin{figure}[t]
\centering
\resizebox{0.41\textwidth}{!}{%
\begin{tikzpicture}[font=\fontsize{10}{13}\selectfont]
\draw[thick] 
    (0,2) node[anchor=east] {$-50^\circ, S_0$} -- (0,0) node[anchor=east] {$-50^\circ, S_0$}
    -- (3,0) node[anchor=south, yshift=-0.5mm, xshift=-16mm] {$\quad\mathit{rup}(P_1, 5)\;\;$} node[anchor=west] {$\varphi_2 = -50^\circ, \gamma_2, S_1$}
    -- (3,-2) node[anchor=east]  {$300^\circ, \gamma_2, S_1$}
    -- (6,-2) node[anchor=south, yshift=-0.5mm, xshift=-15mm] {$\mathit{noOp}(15)$} node[anchor=west] {$300^\circ, \gamma_1, S_2$}
    -- (6,-4) node[anchor=east]  {$475^\circ, \gamma_1, S_2$}
    -- (9.5,-4) node[anchor=south, yshift=-0.5mm, xshift=-17mm] {$\mathit{mRad}(P_1, 20)\;$} node[anchor=west] {$475^\circ, \gamma_1, S_3$}
    -- (9.5,-6) node[anchor=east]  {$615^\circ, \gamma_1, S_3$}
    -- (12.5,-6) node[anchor=south, yshift=-0.5mm, xshift=-13mm] {$\mathit{fixP}(P_1, 26)\quad$} node[anchor=west] {$615^\circ, \gamma_3, \sigma_2'$};

\filldraw [black] (0,0) circle (2pt);
\filldraw [black] (0,2) circle (2pt);
\filldraw [black] (3,0) circle (2pt);
\filldraw [black] (3,-2) circle (2pt);
\filldraw [black] (6,-2) circle (2pt);
\filldraw [black] (6,-4) circle (2pt);
\filldraw [black] (9.5,-4) circle (2pt);
\filldraw [black] (9.5,-6) circle (2pt);
\filldraw [black] (12.5,-6) circle (2pt);
\end{tikzpicture}
}
    \caption{Defused Situation $\sigma_2'$}
    \label{fig:cf1}
\end{figure}

\section{Discussion and Conclusion}\label{sec13}
We presented a first account of actual cause in a proper hybrid action-theoretic framework and defined causation using both a foundational approach, where we studied actions and their effects to define achievement causation, and a production approach, where we defined causation through contributions. One of the key challenges in the study of actual causation is that the intuitive soundness of these proposals cannot be demonstrated, except perhaps using examples \cite{Halpern16}. This leads to the development of definitions by analyzing a handful of simple examples, and then validating relative to our intuition for these examples, a process which G\"{o}\ss ler et al. \cite{GosslerSS17} referred to as ``TEGAR'' (i.e.\ Textbook Example Guided Analysis Refinement). Example driven development however historically led to problematic definitions that continue to suffer from various conceptual problems, even after multiple rounds of revisions. This might still be true for the general case; however for our restricted framework where we only consider achievement causation relative to a linear scenarios, we were able to show that a modified but-for test can be used to intuitively and formally justify our definition.
\par
To the best of our knowledge, the only other work linking actual causation and continuous change is 
Halpern and Peters' \cite{HalpernPeters22} generalized structural equations models (GSEM), which do provide notable advantages over SEM, such as the ability to represent variables indexed by time or other continuous parameters. They also allow specifying which interventions are permitted, making them more expressive than traditional structural equations models. However, the approach is still constrained by the inherent expressive limitations of SEM, as it lacks grounding in a robust action theory. 
Also, they did not provide a definition of actual cause in GSEM. Closely related to our work is our preliminary study in \cite{MK2024}, where we also gave a definition of primary cause in the HTSC by distinguishing several causal cases. Our current proposal refines that definition by unifying all such cases into a single and coherent formulation. In addition, 
we introduce a second production-based definition of actual cause and justify it counterfactually using a modified but-for test.
\par
Our current proposal is nonetheless limited in several ways. For instance, we did not address compound effects and also did not consider indirect causes. However, this attempt demonstrates that determining causes requires careful modeling and reasoning in hybrid domains, even under such restrictions. In the future, we plan to study direct and indirect causes for arbitrary compound effects, both discrete and temporal. 



\begin{acks}
This work is partially supported by the Natural Sciences and Engineering Research Council of Canada and by the University of Regina.
\end{acks}



\bibliographystyle{ACM-Reference-Format} 
\bibliography{9-hybrid}
\section*{Technical Appendix}
%
\paragraph{\textbf{Dynamic Formulae.}}
The notion of \emph{dynamic formulae} is defined as follows:
\begin{definition}
Let 
$\vec{x}$, $\theta_a$, and $\vec{y}$ resp.\ range over object terms, action terms, and object and action variables. 
The class of \emph{dynamic formulae} $\varphi$ is defined inductively using the following grammar:
%
\[\varphi::=P(\vec{x}) \mid Poss(\theta_a)\mid\mathit{After}(\theta_a,\varphi)\mid\neg\varphi\mid\varphi_1\wedge\varphi_2
\mid\exists\vec{y}.\;\varphi.
\]
\end{definition}
\noindent That is, a dynamic formula (DF) can be a situation-suppressed fluent, a formula that says that some action $\theta_a$ is possible, a formula that some DF holds after some action has occurred, or a formula that can built from other DF using the usual connectives. Note that $\varphi$ can have quantification over object and action variables, but must not include quantification over situations or ordering over situations (i.e.\ $\sqsubset$). 
We will use $\varphi$ for DF. 
\par
$\varphi[\cdot]$ is defined as follows:
\begin{definition}\label{psiSAT}
\begin{small}
\begin{eqnarray*}
&&  \hspace{-7 mm}\varphi[s]\defi
    \begin{cases}      
      P(\vec{x},s) & \!\!\textup{if }\varphi\textup{ is }P(\vec{x})\\
      \mathit{Poss}(\theta_a,s) & \!\!\textup{if }\varphi\textup{ is }\mathit{Poss}(\theta_a)\\
      \varphi'[do(\theta_a,s)] & \!\!\textup{if }\varphi\textup{ is }\mathit{After}(\theta_a,\varphi')\\
      \neg(\varphi'[s]) & \!\!\textup{if }\varphi\textup{ is }(\neg\varphi')\\
      \varphi_1[s]\wedge\varphi_2[s] & \!\!\textup{if }\varphi\textup{ is }(\varphi_1\wedge\varphi_2)\\
      \exists\vec{y}.\;(\varphi'[s]) & \!\!\textup{if }\varphi\textup{ is }(\exists\vec{y}.\;\varphi')\\
    \end{cases}      
\end{eqnarray*}
\end{small}
\end{definition}
\par
\paragraph{\textbf{KL formalization of Actual Cause \cite{KL21}.}}
\begin{definition}[Primary Cause \cite{KL21}]\label{PCause}
\begin{small}
\begin{eqnarray*}
&&\hspace{-7 mm}\mathit{CausesDir}(a,ts,\varphi,s)\defi
\exists s_a.\;\mathit{timeStamp}(s_a)=ts\wedge(S_0<do(a,s_a)\leq s)\\
&&\hspace{3 mm}\mbox{}\wedge\neg\varphi[s_a]\wedge\forall s'.(do(a,s_a)\leq s'\leq s\supset\varphi[s']).
\end{eqnarray*}
\end{small}
\end{definition} 
\noindent That is, $a$ executed at time-stamp $ts$ is the \emph{primary cause} of effect $\varphi$ in situation $s$ 
iff $a$ was executed in a situation with time-stamp $ts$ in scenario $s$, $a$ caused $\varphi$ to change its truth value to true, and no subsequent actions on the way to $s$ falsified $\varphi$. 
\par
Now, note that a (primary) cause $a$ might have been non-execut\-able initially. Also, $a$ might have only brought about the effect conditionally and this context condition might have been false initially. Thus earlier actions on the trace that contributed to the preconditions and the context conditions of a cause must be considered as causes as well. KL introduced $\mathit{Causes}(a,ts,\varphi,s)$ to inductively capture both primary and indirect causes.
%
The following definition captures both primary and indirect causes.\footnote{In this, we need to quantify over situation-suppressed DF. Thus we must encode such formulae as terms and formalize their relationship to the associated SC formulae. This is tedious but can be done essentially along the lines of \cite{GiacomoLL00}. We assume that we have such an encoding and use formulae as terms directly.} 
\begin{definition}[Actual Cause \cite{KL21}]\label{ACause}
\begin{small}
\begin{eqnarray*}
&&\hspace{-7 mm}\mathit{Causes}(a,ts,\varphi,s)\defi\\
&&\hspace{-7 mm}
\forall P.[
\forall a,ts,s,\varphi.(\mathit{CausesDir}(a,ts,\varphi,s)\supset P(a,ts,\varphi,s))\\
&&\hspace{13 mm}\mbox{}\wedge\forall a,ts,s,\varphi.( \exists a'\!,ts'\!,s'\!.(\mathit{CausesDir}(a'\!,ts'\!,\varphi,s)\\
&&\hspace{29 mm}\mbox{}\land\mathit{timeStamp}(s')\!=\!ts'\land s'<s\\ 
&&\hspace{29 mm}\mbox{}\land
P(a,ts,[\mathit{Poss}(a')\wedge\mathit{After}(a',\varphi)],s')
\\
&&\hspace{15 mm}\mbox{}\supset P(a,ts,\varphi,s))\\
&&\hspace{-2 mm}]\supset P(a,ts,\varphi,s).
\end{eqnarray*}
\end{small}
\end{definition}
\noindent Thus, $\mathit{Causes}$ is defined to be the least relation $P$ such that if $a$ executed at time-step $ts$ directly causes $\varphi$ in scenario $s$ then $(a,ts,\varphi,s)$ is in $P$, and if $a'$ executed at $ts'$ is a direct cause of $\varphi$ in $s$, the time-stamp of $s'$ is $ts'$, $s'<s$, and $(a,ts,[\mathit{Poss}(a')\wedge\mathit{After}(a',\varphi)],s')$ is in $P$ (i.e.\ $a$ executed at $ts$ is a direct or indirect cause of $[\mathit{Poss}(a')\wedge\mathit{After}(a',\varphi)]$ in $s'$), then $(a,ts,\varphi,s)$ is in $P$. Here the effect $[\mathit{Poss}(a')\wedge\mathit{After}(a',\varphi)]$ requires $a'$ to be executable and $\varphi$ to hold after $a'$.
\paragraph{\textbf{Proof Sketches of Properties in Section \ref{HTSCPropoerties}.}}
\begin{lemma}[Uniqueness of Direct Cause]\label{UniqDirCause}
Given a causal setting $\langle \mathcal{D}, \varphi, s \rangle$, it follows that:
\par\vspace{-3 mm}
\begin{small}
\begin{eqnarray*}
&&\hspace{-5 mm}\mathcal{D} \models \forall a, a'\!, ts, ts'\!.\; \mathit{CausesDir}(a,ts,\varphi,s) \land \mathit{CausesDir}(a',ts',\varphi,s) \\
&& \hspace{10 mm}\supset a = a' \land ts = ts'.
\end{eqnarray*}
\end{small}
\end{lemma}
\begin{proof}[\textbf{Proof}]
Follows directly from Definition \ref{PCause}.
\end{proof}
%
\begin{lemma}[Uniqueness of Achievement Situation]\label{UniqAchvSit}
Given a hybrid setting $\langle \mathcal{D}, \varphi, s \rangle$, we have:
\par\vspace{-3 mm}
\begin{small}
\begin{eqnarray*}
&&\mathcal{D} \models \mathit{AchvSit}(s_\varphi, \varphi, s) \land \mathit{AchvSit}(s'_\varphi, \varphi, s) \supset s_\varphi = s'_\varphi.
\end{eqnarray*}
\end{small}
\end{lemma}
\begin{proof}[\textbf{Proof}] Follows directly from Definition \ref{DefAchvSit}.
\end{proof}
%
\begin{theorem}[Uniqueness of Primary Cause of Temporal Effects]\label{Property-Uniqueness}
Given a hybrid setting $\langle \mathcal{D}, \varphi, s \rangle$, we have:
\par\vspace{-3 mm}
\begin{small}
\begin{eqnarray*}
&&\hspace{-5 mm}\mathcal{D} \models \mathit{CausesDir}^\mathit{prim}_\mathit{temp}(a_1, ts_1, \varphi, \sigma) \land \mathit{CausesDir}^\mathit{prim}_\mathit{temp}(a_2, ts_2, \varphi, \sigma) \\
&& \hspace{10 mm} \supset a_1 = a_2 \land ts_1 = ts_2.
\end{eqnarray*}
\end{small}
\end{theorem}
%
\begin{proof}[\textbf{Proof}]
Follows from the definition of primary achievement cause (Definition \ref{PTI}), the uniqueness of the achievement situation $s_\varphi$ (Lemma \ref{UniqAchvSit}), the mutual exclusivity of contexts (the Mutex Axiom in Section \ref{prelim}), and the uniqueness of direct causes (Lemma \ref{UniqDirCause}).
\end{proof}
%
%
%
\par
%
\begin{theorem}[Implicit Primary Cause]\label{NoCause}
Assume that $\varphi$ is a constraint on the values of a primitive temporal fluent $f$. Then: 
\par\vspace{-3 mm}
\begin{small}
\begin{eqnarray*}
&& \hspace{-7 mm} \mathcal{D} \models (\mathit{ProperHTSCAchvCausalSetting}(\varphi, \sigma)\land \exists s_\varphi.\; \mathit{AchvSit}(s_\varphi, \varphi, \sigma) \\
&& \hspace{0 mm} \mbox{}\land \exists i.\; \gamma_i^f[s_\varphi] \land (\forall s'.\; S_0 \leq s' \leq s_\varphi \supset \gamma_i^f[s']) \\
&& \hspace{0 mm} \supset \neg \exists a, ts.\; \mathit{CausesDir}^\mathit{prim}_\mathit{temp}(a, ts, \varphi, \sigma)),\\
&& \hspace{-7 mm} \text{where,}\\ 
&& \hspace{-7 mm}\mathit{ProperHTSCAchvCausalSetting}(\varphi, \sigma) \defi \mathit{Exec}(\sigma)
\land \exists a_0.\; do(a_0, S_0) \leq \sigma\\
&& \hspace{0 mm}{}\land \neg \varphi[\mathit{start}(S_0), S_0] \land \neg \varphi[\mathit{time}(a_0), S_0] \land \varphi[\mathit{start}(\sigma), \sigma].
\end{eqnarray*}
\end{small}
\end{theorem}
%
%
\begin{proof}[\textbf{Proof (by contradiction)}]
Fix $\varphi_1, \sigma_1, s_{\varphi_1}, \gamma_{i_1}^f$ and assume that:
\begin{eqnarray}
&& \hspace{0 mm} \mathit{AchvSit}(s_{\varphi_1}, \varphi_1, \sigma_1) \land \gamma_{i_1}^f[s_{\varphi_1}] \label{4.41},\\
&& \hspace{0 mm} \forall s'. S_0 \leq s' \leq s_{\varphi_1} \supset \gamma_{i_1}^f[s'] \label{4.42}.
\end{eqnarray}
Fix $a_1$ and $ts_1$ and also assume that: 
\begin{eqnarray}
&& \hspace{0 mm} \mathit{CausesDir}_\mathit{temp}^\mathit{prim}(a_1, ts_1, \varphi_1, \sigma_1) \label{4.43}.
\end{eqnarray}
Now, note that by \ref{4.41} and Lemma \ref{UniqAchvSit}, the achievement situation $s_{\varphi_1}$ is unique. Thus by the Mutex Axiom (which guarantees that the contexts are mutually exclusive) and the definition of primary achievement cause in primitive temporal case (Definition \ref{PTI}), we have:
\begin{eqnarray}
\mathit{CausesDir}(a_1, ts_1, \gamma_{i_1}^f, s_{\varphi_1}).
\end{eqnarray}
But this is contradictory to \ref{4.42} and the definition of direct cause (Definition \ref{PCause}).
\end{proof}
%
%
\begin{lemma}\label{PersistenceOfAchvSit} Given a hybrid setting $\langle \mathcal{D}, \varphi, s \rangle$, we have:
\begin{small}
\begin{eqnarray*}
&&\hspace{-7 mm}\mathcal{D} \models \mathit{AchvSit}(s_\varphi, \varphi, s) \land s < s^*\\
&&\hspace{-5 mm} \mbox{}\land (\forall s', t.\; s \leq s' \leq s^* \land \mathit{start}(s') \leq t \leq \mathit{end}(s', s^*) \supset \varphi[t, s'])\\
&&\hspace{0 mm} \mbox{} \supset \mathit{AchvSit}(s_\varphi, \varphi, s^*).
\end{eqnarray*}
\end{small}
\end{lemma}
\begin{proof}[\textbf{Proof}]
Follows from antecedent and Definition \ref{DefAchvSit}.
\end{proof}
\noindent 
%
\begin{theorem}[Persistence]\label{PCausePersistence}Given 
setting $\langle \mathcal{D}, \varphi, s \rangle$, we have:
\begin{small}
\begin{eqnarray*} 
&&\hspace{-7 mm}\mathcal{D} \models \mathit{CausesDir}^\mathit{prim}_\mathit{temp}(a, ts, \varphi, s) \\
&&\hspace{-5 mm}\mbox{}\land (\forall s', t'\!.\; s \leq s' \leq s^* \land \mathit{start}(s') \leq t' \leq \mathit{end}(s'\!, s^\ast) \supset \varphi[t'\!, s']) \\
&&\hspace{0 mm}\supset \mathit{CausesDir}^\mathit{prim}_\mathit{temp}(a, ts, \varphi, s^*).
\end{eqnarray*}
\end{small}
\end{theorem}
%
%
\begin{proof}[\textbf{Proof}]
By Lemma \ref{PersistenceOfAchvSit} and the antecedent, the achievement situation $s_\varphi$ in $s$ and $s^\ast$ remains the same. The Property thus follows from this, the definition of Primary Achievement Cause in Primitive Temporal Case (Definition \ref{PTI}), the uniqueness of achievement situations (Lemma \ref{UniqAchvSit}), the mutual exclusivity of contexts (the Mutex Axiom), and the uniqueness of direct cause (Lemma \ref{UniqDirCause}), which together ensures that the direct cause of the unique context associated with the unique achievement situation $s_\varphi$ in $s$ and $s^\ast$ is unique.
\end{proof}
\paragraph{\textbf{Proof Sketches of Properties in Section \ref{Contrib}.}}
\begin{theorem}[Equivalence of Primary Cause and Direct Cause Definitions]\label{thm:equivalence}\
Given a hybrid setting $\langle\mathcal{D}, \sigma, \varphi\rangle$, we have:
\begin{eqnarray*}
&&\hspace{-7 mm}\mathcal{D}\models\forall\alpha_1,\alpha_2,ts_1,ts_2.\;(\mathit{CausesDir}_\mathit{temp}^\mathit{prim}(\alpha_1, ts_1, \varphi, s)\\ 
&&\hspace{10 mm}{}\land\mathit{PrimCause}(\alpha_2,ts_2,\varphi,\sigma))
\supset(\alpha_1 = \alpha_2 \land ts_1 = ts_2).
\end{eqnarray*}
\end{theorem}
%
%
\begin{proof}[\textbf{Proof sketch}]
By Lemma \ref{UniqAchvSit} and the Mutex Axiom, the achievement situation $s_\varphi$ and the context $\gamma_i^f$ enabled is $s_\varphi$ are the same in both definitions. According to Definition \ref{PTI}, $\alpha_1$ executed in $ts_1$ directly causes $\gamma_i^f$ in $s_\varphi$. In Definition \ref{CF-PrimaryCause}, $\alpha_2$ executed in $ts_2$ directly causes the same $\gamma_i^f$ in $s_\varphi$ (see Definitions \ref{CF-DirPossContr} and \ref{CF-DirActContr}). Given the uniqueness of direct causes (Lemma \ref{UniqDirCause}), $\alpha_1$ and $\alpha_2$ must be the same action and their execution time must also be the same.
\end{proof}
\paragraph{\textbf{Proof Sketches of Properties in Section \ref{ButFor}.}}
\begin{theorem}\label{DisCF}[Preempted Contribution]
    \begin{small}
    \begin{eqnarray*}
        &&\hspace{-7 mm}\mathcal{D}\nvDash \mathit{Causes}(a,\mathit{ts},\phi,s)\supset
        \neg\exists s'.\;\mathit{CFEx}_\mathit{one}(s',s,\langle\mathit{noOp},a,ts\rangle)\vee\neg\phi(s').
    \end{eqnarray*}
    \end{small}
\end{theorem}
%
%
%
\begin{proof}[\textbf{Proof Sketch}]
This can be proven using a counter-example, e.g.\ using our running example's discrete variant with domain theory $\mathcal{D}_\mathit{npp}^\mathit{SC}$, where replacing the direct cause in setting $\langle \mathcal{D}_\textit{npp}^\mathit{SC}, \varphi_3, \sigma_3 \rangle$ leads to an executable situation where $\varphi_3$ still holds, where $\varphi_3 = \mathit{Ruptured}(P_1,\sigma_2)$ and $\sigma_2 = do([\mathit{rup}(P_1),\mathit{mRad}(P_1), \mathit{fixP}(P_1),\mathit{rup}(P_1),$ $\mathit{rup}(P_1)], S_0)$.
\end{proof}
\begin{lemma}\label{Lemma-ExisOfMaxPreemp}
\begin{eqnarray*}
&&\hspace{-7 mm}\mathcal{D} \models \exists a, ts.\; \mathit{PrimCause}(a, ts, \varphi, \sigma) \supset \exists \sigma'\!.\; \mathit{DefusedSit}(\varphi, \sigma, \sigma').
\end{eqnarray*}
\end{lemma}
\begin{proof}[\textbf{Proof}]
By Definition \ref{CF-PrimaryCause} and \ref{CF-PreemptedContributors}, if $a$ is a primary cause then $\sigma$ must have been a non-initial situation(i.e., $\sigma \neq S_0$), and thus we can always construct $\sigma'$ by replacing $a$ with the $\mathit{noOp}(time(a))$ action in $\sigma$.
\end{proof}
\begin{lemma}\label{UniquePreepContr}
\begin{eqnarray*}
&&\hspace{-7 mm} \exists a_1, ts_1, \sigma_1, \varphi, \sigma, a_2, ts_2, \sigma_2, \varphi, \sigma.\; \mathit{PreempContr}(a_1, ts_1, \sigma_1, \varphi, \sigma) \\
&&\hspace{0 mm}\mbox{}
\land \mathit{PreempContr}(a_2, ts_2, \sigma_2, \varphi, \sigma) \land |\sigma_1| = |\sigma_2| \supset \sigma_1 = \sigma_2.
\end{eqnarray*}
\end{lemma}
\begin{proof}[\textbf{Proof sketch (by induction)}]
We start by induction on $|\sigma_1|$, the number of $\mathit{noOp}$ actions in $\sigma_1$.
For $n = 1$, both $\sigma_1^b$ and $\sigma_2^b$ contain exactly one $\mathit{noOp}$ action, which corresponds to a substitution of the primary cause of $\varphi$ in $\sigma$. By Definition \ref{CF-PreemptedContributors}, Theorem \ref{Property-Uniqueness} (the uniqueness of the primary cause), and Theorem \ref{thm:equivalence} (that the two definitions of primary causes are equivalent), we have:
\begin{eqnarray*} 
&&\hspace{-7 mm}\mathit{PreempContr}(a_1, ts_1, \sigma_1^b, \varphi, \sigma) \land \mathit{PreempContr}(a_2, ts_2, \sigma_2^b, \varphi, \sigma)\\
&&\hspace{-7 mm}{}\supset \sigma_1^b = \sigma_2^b. 
\end{eqnarray*}
Assume that the consequence holds for $|\sigma_1^k| = |\sigma_2^k| = k$. We will show that it holds for $|\sigma_1| = |\sigma_2| = k+1$. By the same argument as in the base case (i.e., that replacing the actual cause from the same situation yields a unique situation due to the uniqueness of primary cause), it can be shown that $\sigma_1 = \sigma_2$.
\end{proof}
%
\begin{lemma}\label{UniqueMaxPreepContr}
\begin{eqnarray*}
&&\hspace{0 mm} \exists \sigma'\!, \sigma''\!.\; \mathit{DefusedSit}(\varphi, \sigma, \sigma') \land \mathit{DefusedSit}(\varphi, \sigma, \sigma'') \supset \sigma' = \sigma''.
\end{eqnarray*}
\end{lemma}
\begin{proof}[\textbf{Proof}]
Follows directly from Lemma \ref{UniquePreepContr}, Definition \ref{CF-MaxPreempContrCFSit}, and the fact that since the scenario $\sigma$ is finite, only a finite number of actions could be replaced with $\mathit{noOp}$ actions.
\end{proof}
\begin{lemma}\label{NoCauseInMPreempSit}
\begin{eqnarray*}
&&\mathit{DefusedSit}(\varphi, s, s') \supset \neg\exists b, ts_b. \mathit{PrimCause}(b, ts_b, \varphi, s').
\end{eqnarray*}
\end{lemma}
\begin{proof}[\textbf{Proof (by contradiction)}]
Fix $\varphi_1, s_1, s_1'$ and assume on the contrary that for some $b_1$ and $ts_{b_1}$,
\[\mathit{PrimCause}(b_1, ts_{b_1}, \varphi_1, s_1').\]
According to Definition \ref{CF-PreemptedContributors}, 
\[\mathit{PreempContr}(b_1, ts_{b_1}, s_1'', \varphi, s_1'),\] 
where $s_1''$ is a counterfactual situation of $s_1'$ with $b_1$ replaced by $\mathit{noOp(time(b_1))}$. This implies $|s_1'| < |s_1''|$, contradicting Definition \ref{CF-MaxPreempContrCFSit} and Lemma \ref{UniqueMaxPreepContr}, which requires $s_1'$ to be the unique defused situation and $s_1'$ to have the maximum number of $\mathit{noOp}$ actions. Therefore, no such $b_1$ exists, proving the lemma.
\end{proof}
\begin{theorem}[Counterfactual Dependence]\label{CF-Dep-Therorem}
Given a hybrid setting, $\langle\mathcal{D},\sigma,\varphi\rangle$, where $\varphi$ is a constraint on the values of $f$, we have:
\par\vspace{-3 mm}
\begin{small}
\begin{eqnarray*}
&&\hspace{-7 mm} \exists a, ts.\;\mathit{PrimCause}(a, ts, \varphi, s) \supset{} 
\\
&&\hspace{-7 mm}[\exists s'\!.\;\!\mathit{DefusedSit}(\varphi, s, s')\!\land\!(\wedge_{i} \neg\gamma_i^f[S_0] \supset    \neg(\varphi[start(s'), s']\!\land\!\mathit{Exec}(s')))].
\end{eqnarray*}
\end{small}
\end{theorem}
%
\begin{proof}[\textbf{Proof sketch}] 
Fix $A_1, ts_1, \varphi_1, s_1$ and assume that $\mathit{PrimCause}(A_1,$ $ts_1, \varphi_1, s_1)$.
From Lemma \ref{Lemma-ExisOfMaxPreemp}, it follows that there is a situation, say $s_1'$, such that $\mathit{DefusedSit}\allowbreak(\varphi_1, s_1, s_1')$.
\par
Now, assume that $\bigwedge_{i} \neg\gamma_i^f[S_0]$. Also assume that $\mathit{Exec}(s_1')$. We need to show that $\neg\varphi_1[start(s_1'), s_1']$. We will prove this by contradiction. Assume that $\varphi_1[start(s_1'), s_1']$. Since all the contexts of the only fluent $f$ in $\varphi_1$ were initially false, there must be an action $A_1'(t)$ executed at some timestamp $ts'$ that brought about some context of $\varphi$ which eventually achieved $\varphi$. Consider the situation that replaces $A_1'(t)$ with $\mathit{noOp}(t)$ in $s_1'$, let's call it $s_1^*$. By Definition \ref{CF-PreemptedContributors}, $s_1^*$ must be a preempted contributor, i.e., $\mathit{PreempContr}(A_1'(t), ts', s_1^*, \varphi, \sigma)$. 
Moreover, $|s_1^*| > |s_1'|$; but this means that by Definition \ref{CF-MaxPreempContrCFSit}, $s_1'$ cannot be the defused situation with respect to $\varphi$ and $s_1$, which is contradictory to the above assumption.
\end{proof}

\end{document}